\documentclass[runningheads]{llncs}

\usepackage[T1]{fontenc}
\usepackage[utf8]{inputenc}
\usepackage{graphicx}
\usepackage{amsmath}
\usepackage{amssymb}
\usepackage{booktabs}
\usepackage{placeins}
\usepackage{url}
\usepackage[hidelinks]{hyperref}
\usepackage{xspace}

\newcommand{\method}{\textsc{Caliber}\xspace}

\begin{document}

\title{\textsc{Caliber}: Cross-Architecture Extraction-Cost Control for Score-Returning APIs}

\titlerunning{\textsc{Caliber}: Cross-Architecture Extraction-Cost Control}

\author{Chi Wang\inst{1} \and Hanwen Wang\inst{2} \and Yu Xia\inst{3} \and
Zihan Wang\inst{1} \and Guangdong Bai\inst{4}}

\authorrunning{C. Wang et al.}

\institute{University of Queensland, Brisbane, Australia\\
\email{chi.wang2@student.uq.edu.au, zihan.wang@uq.edu.au} \and
Adelaide University, Adelaide, Australia\\
\email{hanwen.wang@adelaide.edu.au} \and
Beijing Technology and Business University, Beijing, China\\
\email{2307010124@st.btbu.edu.cn} \and
City University of Hong Kong, Hong Kong, China\\
\email{g.bai@cityu.edu.hk}}

\maketitle

\begin{abstract}
We present \method, an output-perturbation defense against model extraction that formulates noise selection as a calibration problem: how much the defense degrades the supervision signal used to train a surrogate, and the provable per-input query cost of recovering the clean logits. To defend against an attacker that uses returned scores for knowledge distillation, \method adds independent and identically distributed Gaussian noise to the internal logits. We establish two properties of the resulting perturbed predictions. \emph{Monotone agreement degradation:} When the clean logits have a unique maximizer, agreement with the clean prediction decreases strictly with the noise scale, so every target in $(1/K,1)$ corresponds to a unique positive scale; task accuracy is bounded by computable lower and upper envelopes. \emph{Per-input recovery cost:} We derive a closed-form minimax lower bound on the repeated queries needed to recover the clean logits for a fixed input. \method normalizes noise variance by the squared median top-two logit margin and fits the resulting noise-utility relationship with a logistic curve, either per model or shared within a task. Across more than thirty model-dataset combinations, per-model calibration achieves mean absolute relative errors of 0.6--1.4\%. End-to-end experiments show that surrogate performance generally tracks the configured degradation, while fixed-input averaging follows the expected variance reduction.

\end{abstract}

\section{Introduction}

Deep learning models are expensive to train, which makes them valuable intellectual property worth protecting. Training Llama 3 405B consumed 30.84 million GPU hours on up to 16,000 H100 GPUs~\cite{grattafiori2024llama}. For score-returning APIs, confidence vectors improve downstream usefulness but also provide a supervision signal that an attacker can use for distillation. Querying at scale, an attacker collects the returned scores and trains a surrogate by knowledge distillation~\cite{hinton2015distilling} that approximates the model's behavior at a fraction of the original cost~\cite{tramer2016stealing,orekondy2019knockoff}. Such extraction has already been demonstrated against production language-model APIs~\cite{krishna2020thieves,carlini2024stealing}, and Anthropic reported coordinated distillation campaigns involving more than 16 million API
interactions from approximately 24,000 accounts whose operators sought to reproduce Claude's capabilities~\cite{anthropic2026distillation}.

Blocking suspicious accounts is difficult at the scale of a commercial API service. Existing detectors identify query patterns that deviate from benign traffic~\cite{kesarwani2018extraction,juuti2019prada,zhang2021seat}, yet heavy legitimate use and distributed extraction can produce overlapping query volumes and input distributions. Providers therefore maintain conservative false-positive rates, allowing some malicious activity to go undetected. Blocking also offers no intermediate response between unrestricted access and complete denial.

An alternative is to keep a flagged account active while limiting the information that its user can extract. The provider perturbs the confidence scores it returns, corrupting the soft labels that distillation needs at a calibrated accuracy cost while each returned response stays plausible~\cite{lee2019defending,orekondy2020prediction,kariyappa2021protecting}. This caps the accuracy of the resulting surrogate, and an attacker who wants the clean scores back must average repeated queries on the same input, paying a multiple of the query budget. Existing output-perturbation defenses require model-specific tuning. Because the map from noise magnitude to utility loss depends on the architecture, a raw noise level that barely dents one model may severely degrade another, so unnormalized settings do not transfer across a heterogeneous portfolio. A practical defense should instead map a specified utility target to the model-specific noise that reaches it.

\begin{figure}[!ht]
  \centering
  \includegraphics[width=1\textwidth]{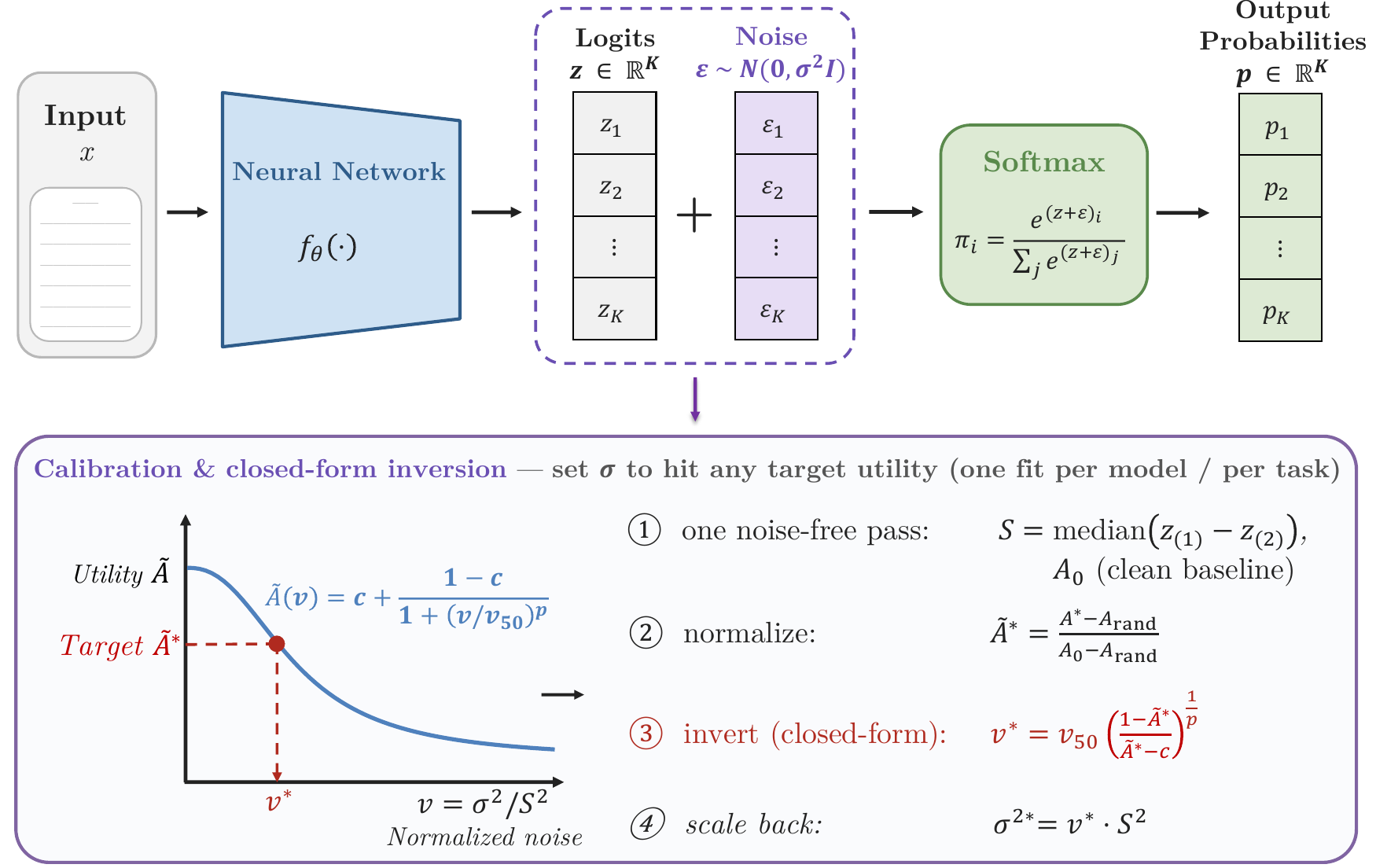}
  \caption{Overview of \method{}. Top: i.i.d.\ Gaussian noise of variance $\sigma^2$ is injected into the logits before the softmax. Bottom: after margin normalization, the noise-utility relationship is a single logistic curve, which \method{} inverts in closed form to map a target utility $A^*$ to the required $\sigma^2$, fit per model or shared per task.}
  \label{fig:unidial-overview}
\end{figure}

\method{} is that mechanism. It answers \emph{how a provider sets the degradation of returned scores on any served model}, capping the supervision an extraction campaign can distill and forcing an attacker who tries to cancel the zero-mean noise on an input to pay a multiplicative query cost that we bound below. As Figure~\ref{fig:unidial-overview} shows, \method injects independent and identically distributed (i.i.d.)\ Gaussian noise of variance $\sigma^2$ into the model logits. The noise required for a given utility reduction varies substantially across architectures, and much of this variation is captured by a model-level scale $S$, defined as the median top-two logit margin over a validation set. Once the noise is normalized by this margin, every architecture we evaluate, spanning CNNs, Vision Transformers, large language models (LLMs), and vision-language models (VLMs), follows the same monotone logistic law. Our contributions are as follows:
\begin{enumerate}
\item We recast extraction defense as a controllable problem, proving strictly monotone degradation with a computable accuracy corridor and a minimax lower bound on the queries needed to cancel the injected noise on any single input. The cost is deliberately per-input, and we mark where the empirical evidence takes over.

\item We introduce a margin-normalized calibration law that sets any degradation target on a served model or on an unseen architecture, without per-model tuning.

\item We validate \method{} across more than thirty model-dataset combinations, hitting utility targets to within ${\sim}1\%$. Furthermore, our end-to-end extraction study demonstrates that the stolen surrogate inherits the calibrated degradation.
\end{enumerate}

\section{Problem Formulation}
\subsection{Threat Model}
\label{sec:threat-model}

\begin{figure}[t]
  \centering
  \includegraphics[width=0.65\textwidth]{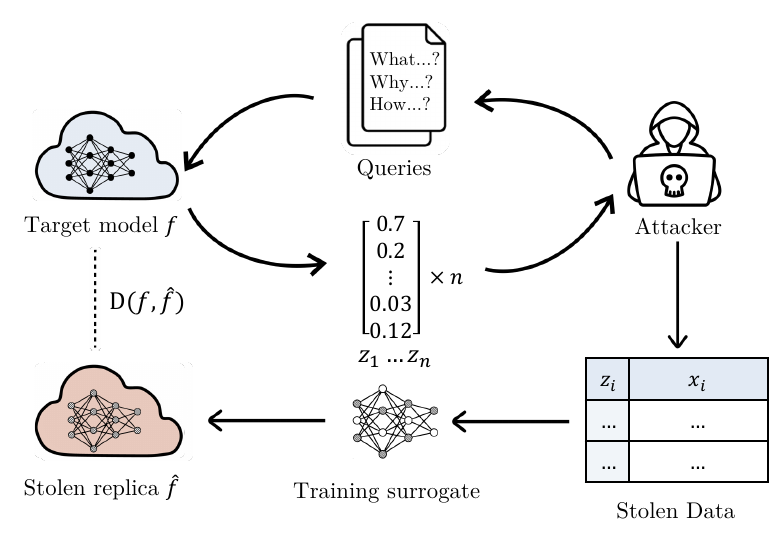}
  \caption{Model extraction threat model}
  \label{fig:threat-model}
\end{figure}

\paragraph{Assets and attacker goal.} A provider serves a target model $f$ through a black-box API. Without access to the model's parameters, training data, or architecture, the attacker seeks a surrogate $\hat{f}$ that minimizes the expected output discrepancy over a task-relevant input distribution $P_X$, namely $\mathbb{E}_{x\sim P_X}[D(f(x),\hat{f}(x))]$, where $D$ is defined below.

\paragraph{Attacker capabilities.} The attacker holds a standard API subscription and submits arbitrary queries $x \in \mathcal{X}$, each returning a confidence score vector $\boldsymbol{z}=f(x)\in\mathbb{R}^K$ (the logits) over $K$ candidate outputs. The attacker knows the task domain but not $f$'s architecture or training. This interface matches deployed systems, where production LLM APIs expose top-$k$ log-probabilities and specialized fine-tuned classifiers are served through the same score-returning endpoints.

\paragraph{Attack mechanism.} The attacker collects $\mathcal{D}_{\mathrm{steal}} = \{(x_i, \boldsymbol{z}_i)\}_{i=1}^{N}$ and trains $\hat{f}$ by knowledge distillation (KD)~\cite{hinton2015distilling}, minimizing the KD loss
\begin{equation}
  \mathcal{L}_{\mathrm{KD}} = \sum_{i=1}^{N} \operatorname{KL}\!\left(\operatorname{softmax}\!\left(\frac{\boldsymbol{z}_i}{T}\right) \,\Big\|\, \operatorname{softmax}\!\left(\frac{\hat{f}(x_i)}{T}\right)\right),
\end{equation}
where $T > 1$ is the distillation temperature. Figure~\ref{fig:threat-model} depicts this query-and-distill pipeline. The full score vector carries inter-class similarity~\cite{hinton2015distilling}, often called the teacher's dark knowledge, which makes surrogate training far more sample-efficient than learning from hard labels alone~\cite{orekondy2019knockoff}.

\paragraph{Defender capabilities and constraints.} The provider controls the endpoint, sees every query with its metadata (account ID, IP, timing), and has full access to its own confidence scores at inference. It cannot reliably separate extraction traffic from legitimate heavy use, whose query volumes and distributions overlap with distillation campaigns~\cite{juuti2019prada}.

\paragraph{Defense goal.} A good defense degrades a flagged account's outputs by a calibrated amount, scrambling the soft-label signal and, past the mildest settings, flipping a controlled fraction of predicted labels. Benign accounts receive unmodified outputs. Because a mistaken flag now perturbs rather than bans, the provider can run detection more aggressively than any ban policy allows. The difficulty is \emph{calibration}. The provider must set a target degradation and reliably reach it, perturbing neither so hard that outputs look visibly wrong nor so little that the surrogate stays usable.

\subsection{Defense Requirements}
\label{sec:requirements}
A practical perturbation mechanism needs five properties. It must leave the served weights untouched (\emph{parameter-free}) and add little enough latency to stay under a timing side-channel (\emph{minimal overhead}). Its degradation must be \emph{monotonic}, so that surrogate quality can be capped, and \emph{transferable} enough that one policy covers much of a provider's fleet. Finally, \emph{one-shot calibration} makes the mechanism deployable from a single validation pass.

\section{Methodology}
\label{sec:method}

Models differ in logit scale, but once noise is normalized by the logit margin its effect on utility follows a single monotone logistic law that inverts in closed form, fit either per model or once per task. The i.i.d.\ Gaussian primitive and its order-preservation analysis build on Wang et al.~\cite{wang2025aim}.

\subsection{Logits Perturbation}

Logits, the raw scores before softmax, are an effective point of intervention for controlled degradation. \method injects isotropic Gaussian noise into the logits at inference time, which moves utility without touching a single weight. For input logits \(\boldsymbol{z} = (z_1, z_2, \dots, z_K)\), the perturbed logits are:
\begin{equation}
  \boldsymbol{z}' = \boldsymbol{z} + \boldsymbol{\epsilon}, \quad \epsilon_k \overset{\text{i.i.d.}}{\sim} \mathcal{N}(0, \sigma^2),
  \label{eq:logits-noise}
\end{equation}
where \(\sigma^2\) is the noise variance, the core control parameter, and \(\boldsymbol{\epsilon}\) denotes the noise vector with i.i.d.\ components. The final prediction is derived from the perturbed logits via softmax:
\begin{equation}
  \hat{y}' = \operatorname*{arg\,max}_{k} \operatorname{softmax}(\boldsymbol{z}')_k.
\end{equation}

The impact of noise on utility stems from its disruption of logit ordering~\cite{guo2017calibration,platt1999probabilistic,niculescu2005predicting,lakshminarayanan2017simple}. Following the order-preservation analysis of additive Gaussian logit noise~\cite{wang2025aim}, let \(\tau\) be a permutation that sorts the original logits in ascending order, \(z_{\tau(1)} \leq z_{\tau(2)} \leq \dots \leq z_{\tau(K)}\). Define the adjacent logits gap as \(\Delta_i = z_{\tau(i+1)} - z_{\tau(i)}\) for \(i = 1, 2, \dots, K-1\). The probability that the perturbed logits retain the original ordering is approximated by:
\begin{equation}
  \Pr(\text{order preserved}) \approx \prod_{i=1}^{K-1} \Phi\left( \frac{\Delta_i}{\sqrt{2}\sigma} \right),
  \label{eq:order-preservation}
\end{equation}
where \(\Phi(\cdot)\) is the cumulative distribution function (CDF) of the standard normal distribution. Each factor is the exact marginal survival probability of one adjacent pair, and the product treats these events as independent. Because adjacent gaps share a noise component, this is an approximation rather than an identity. As \(\sigma\) increases, \(\frac{\Delta_i}{\sqrt{2}\sigma}\) decreases, so \(\Phi(\cdot)\) and the product decline. This product describes survival of the \emph{entire} ranking. For utility, the operative event is survival of the predicted label.

\begin{proposition}[Exact label preservation and monotone degradation]
\label{prop:monotone}
Fix an input whose clean logits have a unique maximizer \(a=\operatorname*{arg\,max}_{k} z_k\), and let \(g_j=z_a-z_j>0\) for \(j\neq a\). Under i.i.d.\ noise \(\epsilon_k\sim\mathcal{N}(0,\sigma^2)\), the probability that the perturbed prediction matches the clean one is
\begin{equation}
  \pi(\sigma)\;:=\;\Pr\!\Big[\operatorname*{arg\,max}_{k}\,(z_k+\epsilon_k)=a\Big]
  \;=\;\int_{\mathbb{R}}\phi(u)\prod_{j\neq a}\Phi\!\Big(\tfrac{g_j}{\sigma}+u\Big)\,du,
  \label{eq:top1-exact}
\end{equation}
where \(\phi\) and \(\Phi\) are the density and CDF of the standard normal. The map \(\pi\) is continuous and strictly decreasing on \((0,\infty)\), with \(\lim_{\sigma\to0^+}\pi(\sigma)=1\) and \(\lim_{\sigma\to\infty}\pi(\sigma)=1/K\).
\end{proposition}

\emph{Proof sketch.} Conditioning on the noise at the top class reduces the event to independent Gaussian comparisons, which gives the integral \eqref{eq:top1-exact}. Since every gap \(g_j>0\), raising \(\sigma\) strictly shrinks the survival region, driving \(\pi\) from \(1\) down to the chance level \(1/K\). The full proof is in Appendix~\ref{app:proofs}.

Averaged over inputs, the population agreement with the clean decision decreases strictly from \(1\) to the chance level \(1/K\). Task accuracy is a distinct quantity. On an input the clean model misclassifies, noise can flip the prediction onto the true label, so accuracy need not fall pointwise in \(\sigma\). It is nevertheless pinned to the agreement curve, with a slack we can bound explicitly.

\begin{corollary}[Accuracy corridor]
\label{cor:corridor}
On an evaluation distribution with clean accuracy \(A_0\), let \(\bar\pi_c(\sigma)\) be the mean of \(\pi(\sigma)\) from Equation~\eqref{eq:top1-exact} over the correctly classified inputs, and for each misclassified input let \(h=z_a-z_y>0\) be the deficit of the true label \(y\) behind the clean prediction \(a\). Then the accuracy under noise, \(A(\sigma)=\operatorname*{avg}_i \Pr[\operatorname*{arg\,max}_k(z_{k,i}+\epsilon_k)=y_i]\), satisfies for every \(\sigma>0\)
\begin{equation}
  A_0\,\bar\pi_c(\sigma)\;\le\;A(\sigma)\;\le\;A_0\,\bar\pi_c(\sigma)+(1-A_0)\,\beta(\sigma),
  \label{eq:corridor}
\end{equation}
where \(\beta(\sigma)=\operatorname*{avg}_{\text{misclassified}}\Phi\!\big(\!-\tfrac{h}{\sqrt{2}\,\sigma}\big)<\tfrac12\), the floor \(A_0\bar\pi_c(\sigma)\) is continuous and strictly decreasing from \(A_0\) to \(A_0/K\), and both envelopes are computable from the same noise-free sweep that fixes \(S\) and \(A_0\). Moreover \(A(\sigma)\to 1/K\) as \(\sigma\to\infty\).
\end{corollary}

Corollary~\ref{cor:corridor}, proved in Appendix~\ref{app:proofs}, ties utility to agreement. Accuracy stays inside a corridor: the floor inherits the strict monotonicity of Proposition~\ref{prop:monotone}, and the width \((1-A_0)\beta(\sigma)\) is computable in closed form from the clean sweep. Across the seventeen CIFAR architectures and eight MMLU LLMs of Section~\ref{sec:experiments}, that width at a setting retaining $85\%$ of clean accuracy is $2$--$3$ accuracy points on CIFAR-10, $6$--$8$ points on CIFAR-100, and $6$--$14$ points on MMLU, scaling with the misclassified mass \(1-A_0\), and measured accuracy never leaves the corridor at any setting we test. Accuracy is what a provider targets, so the calibration below fits it empirically, the corridor guaranteeing that this fit smooths a bounded band around a monotone mechanism rather than an arbitrary curve. Its endpoints \(A_0\) and \(1/K\) are exactly those that normalize utility in Equation~\eqref{eq:util-norm}, and the \(1/K\) floor reappears as the residual term \(c\) in Equation~\eqref{eq:logistic-decay}. When the controlled quantity is agreement itself, no fitting is needed at all. Equation~\eqref{eq:top1-exact} can be inverted numerically per model, and doing so on the cached sweeps of Section~\ref{sec:experiments} hits agreement targets of $0.9$--$0.6$ with a mean absolute relative error of $0.08\%$ ($0.27\%$ at worst) across those same twenty-five CIFAR and MMLU models, the resolution of the $64$-draw Monte Carlo verification itself. This establishes the \emph{monotonic control} requirement of Section~\ref{sec:requirements} at the mechanism level, and Figure~\ref{fig:performance-curves} shows the same monotonicity end-to-end.

\begin{figure}[!t]
  \centering
  \includegraphics[width=1\textwidth]{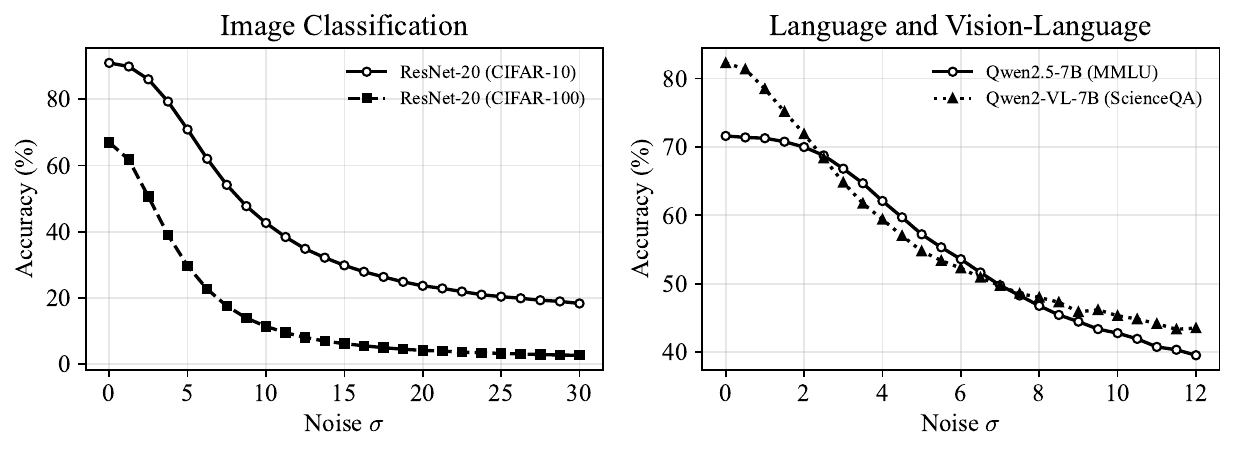}
  \caption{Utility decay under injected logit noise}
  \label{fig:performance-curves}
  \vspace{-0.15cm}
\end{figure}

\subsection{Margin Normalization}

Raw noise variance $\sigma^2$ cannot be compared across models because logit scales differ. ResNet-50 on CIFAR-10 may show logit margins an order of magnitude larger than lightweight MobileNetV2 \cite{sandler2018mobilenetv2}, so the same $\sigma$ causes very different utility drops. We therefore normalize by each model's own logit margin. For each validation sample we take the margin $m_i = z_{(1),i}-z_{(2),i}$, set the model-specific scale factor to its median $S = \operatorname{median}(\{m_i\}_{i=1}^N)$, which is robust to label-ambiguous outliers and ablated against the mean and trimmed mean in Section~\ref{sec:experiments}, and define the normalized noise intensity
\begin{equation}
  v = \frac{\sigma^2}{S^2}.
  \label{eq:normalized-noise}
\end{equation}
A given $v$ then corresponds to comparable logit disruption across models. At $v=1$ the noise variance equals the squared median margin, whatever the architecture. This margin is exactly the smallest gap $\min_{j\neq a} g_j$ from the top logit in Proposition~\ref{prop:monotone}, the gap to the nearest competitor whose corruption first flips the prediction, which is why it sets the relevant noise scale. The margin also delimits the \emph{stealthy} region. While $\sigma$ stays well below it, perturbation scrambles mostly the soft-label tail and usually leaves the predicted label intact, as Equation~\eqref{eq:order-preservation} shows, so $v$ measures how much soft-label corruption a model absorbs before its accuracy starts to fall.

\subsection{Calibration Curve Fitting}
Margin normalization places the noise on a common $v=\sigma^2/S^2$ axis. We likewise normalize utility to each model's operating range,
\begin{equation}
  \tilde{A} = \frac{A - A_{\mathrm{rand}}}{A_0 - A_{\mathrm{rand}}},
  \label{eq:util-norm}
\end{equation}
where $A_0$ is the clean baseline and $A_{\mathrm{rand}}=1/K$ the chance level, so that $\tilde{A}\in[0,1]$ with $\tilde{A}(0)=1$. Across image-classification, LLM, and VLM discriminative tasks the $(v,\tilde{A})$ relationship follows a monotone \emph{logistic} decay. Utility holds near the baseline while the logit margin dominates the noise, drops sharply once the ordering starts to randomize, then settles onto a residual floor:
\begin{equation}
  \tilde{A}(v) = c + \frac{1-c}{1+(v/v_{50})^{p}},
  \label{eq:logistic-decay}
\end{equation}
where $v_{50} > 0$ is the half-degradation scale, $p > 0$ the steepness, and $0 \leq c < 1$ a small residual floor. The logistic form captures the flat head and sharp transition of discrete decision tasks more faithfully than a stretched exponential, which over-predicts degradation at low noise.

\method{} fits $(v_{50},p,c)$ by bounded least-squares on the $(v,\tilde{A})$ points from a single noise-free sweep, in either of two modes. \emph{Per-model} fitting uses the target model's own sweep, the precise deployment mode. \emph{Shared per-task} fitting aggregates a few representative architectures once, after which a new architecture serving that task needs only its margin $S$, with no sweep of its own, trading some precision for zero per-model cost. Measured leave-one-architecture-out, the shared curve stays within a point of the in-sample fit on every task (Section~\ref{sec:experiments}), the residual gap arising because margin normalization equalizes the logit margin but not the full order statistics (Equation~\eqref{eq:order-preservation}), which differ across architectures and class counts $K$. Both modes are monotone, so the inversion below is well-defined.

\subsection{One-Shot Inverse Calculation}
Given a fitted curve, per-model or shared per-task, \method computes in one shot the noise variance that achieves any target utility $A^*$, with no retraining or per-point search. A noise-free pass on a clean validation set fixes the scale factor $S = \operatorname{median}(\{m_i\}_{i=1}^N)$ and baseline $A_0$. We then normalize the target, $\tilde{A}^* = (A^*-A_{\mathrm{rand}})/(A_0-A_{\mathrm{rand}})$, invert Equation~\eqref{eq:logistic-decay}, and de-normalize to the model scale:
\begin{equation}
  v^* = v_{50} \left( \frac{1-\tilde{A}^*}{\tilde{A}^* - c} \right)^{1/p},
  \qquad (\sigma^*)^2 = v^* \, S^2 ,
  \label{eq:inverse-v}
\end{equation}
where $(v_{50},p,c)$ are the fitted parameters. Because $\tilde{A}^*$ decreases as $A^*$ drops, $v^*$ and hence $\sigma^*$ increase monotonically by construction. The inversion is valid for \(c < \tilde{A}^* \leq 1\). For \(A^* \geq A_0\) no noise is needed (\((\sigma^*)^2 = 0\)), and for targets at the chance floor $(\sigma^*)^2$ is set large to saturate degradation.

\subsection{Per-Input Recovery Cost}
\label{sec:cost-guarantee}
Because the injected noise is zero-mean, an attacker can try to cancel it by averaging repeated queries on the same input. The next result bounds what this recovery costs and shows how per-input seeding removes the cheap route.

\begin{proposition}[Minimax query cost of per-input recovery]
\label{prop:cost}
Let the defender operate at normalized intensity \(v=\sigma^2/S^2\) and consider a single input, whose clean logits \(\boldsymbol z\) are unknown to the attacker. The attacker issues \(M\) repeated queries on that input and applies an arbitrary estimator \(\hat{\boldsymbol z}\) to the returned logits. Call the \emph{effective intensity} of a recovery its worst-case normalized error, \(v_{\mathrm{eff}}(\hat{\boldsymbol z})=\sup_{\boldsymbol z}\tfrac{1}{KS^2}\,\mathbb{E}\|\hat{\boldsymbol z}-\boldsymbol z\|^2\).
\begin{enumerate}
  \item[(i)] If the perturbation is drawn freshly i.i.d.\ per query, then every estimator satisfies \(v_{\mathrm{eff}}(\hat{\boldsymbol z})\ge v/M\), with equality attained by the sample mean. Hence to reduce the effective intensity to a usable level \(v_{\mathrm{use}}\in(0,v]\), the attacker must issue \(M\ge v/v_{\mathrm{use}}\) queries per input.
  \item[(ii)] If the perturbation is drawn deterministically from a seed fixed by the input, the \(M\) repeats return identical logits. Their average retains effective intensity \(v\), so exact repetition yields no recovery, and the attacker must instead source \(M\) near-duplicate inputs that the model maps to essentially the same clean logits, each of which draws its own independent perturbation.
\end{enumerate}
\end{proposition}

\emph{Proof sketch.} The averaged responses form a sufficient statistic distributed as \(\mathcal{N}(\boldsymbol z,(\sigma^2/M)I)\), and a limiting Gaussian-prior argument shows no estimator improves on worst-case error \(\sigma^2/M\), which the sample mean attains. The bound is thus minimax over all recovery strategies, not an algebraic property of averaging alone. A per-input deterministic seed instead makes the \(M\) draws identical, so their average yields no variance reduction and recovery must come from distinct inputs. The full proof is in Appendix~\ref{app:proofs}.

Proposition~\ref{prop:cost}(i) turns the deployed intensity into a lower bound \(v/v_{\mathrm{use}}\) on the queries an attacker must spend on a single input to drive that input's residual error down to \(v_{\mathrm{use}}\), holding for every recovery strategy rather than only the sample average. Reading \(v_{\mathrm{use}}\) back through \(\tilde{A}\) as a utility level requires the residual to act like fresh noise of that intensity. That is exact for the sample mean, whose residual is again Gaussian of intensity \(v/M\), and is otherwise a mean-squared-error proxy, since equal-error residuals need not corrupt the soft labels equally.

The scope is narrow. Proposition~\ref{prop:cost} concerns recovery of one input's clean logits from repeated observations of that input, and extraction is not that problem. An attacker fits a surrogate across many distinct inputs, pooling information through the surrogate's inductive bias, which no per-input bound constrains. The \({\sim}M\times\) figure in Section~\ref{sec:extraction} is measured rather than derived, and a lower bound on the query complexity of extraction remains open. What the proposition rules out is the one adaptive strategy that would let an attacker undo the defense for free, and composed with a volume-based detector that caps \(M\) it caps how far per-input recovery can be pushed. Section~\ref{sec:extraction} measures the cross-input route directly and finds that distilling across tens of thousands of distinct perturbed responses attenuates the penalty but does not erase it.

\section{Experiments}
\label{sec:experiments}
We evaluate \method on eight LLMs, one VLM, and twenty-five vision model-dataset combinations spanning eight backbone families and three datasets.
\begin{table}[!t]
  \centering
  \caption{\textbf{LLM and VLM calibration per-model closed-form inversion.}
  Relative error (\%) between target and measured accuracy at targets of
  90/80/70/60\% of each model's clean baseline. LLMs are scored on MMLU, the VLM on ScienceQA.}
  \label{tab:main-nlp-vlm}
  \resizebox{\linewidth}{!}{%
  \begin{tabular}{llccccccc}
    \toprule
    \textbf{Model} & \textbf{Family} & \textbf{Size} & \textbf{Base.} & \multicolumn{4}{c}{\textbf{Rel.\ error at target (\%)}} & \textbf{Mean} \\
     & & & \textbf{(\%)} & \textbf{90\%} & \textbf{80\%} & \textbf{70\%} & \textbf{60\%} & \textbf{$|$err$|$} \\
    \midrule
    Qwen2.5-7B & Qwen & 7B & 71.6 & $+2.8$ & $-0.1$ & $-2.2$ & $-2.1$ & 1.8 \\
    Qwen2.5-14B (4-bit) & Qwen & 14B & 75.9 & $+3.4$ & $+0.6$ & $-2.1$ & $-2.1$ & 2.0 \\
    Llama-3.1-8B & Llama & 8B & 67.1 & $+2.2$ & $+0.9$ & $-0.9$ & $-1.3$ & 1.3 \\
    Mistral-7B & Mistral & 7B & 59.2 & $+2.3$ & $-0.2$ & $-1.3$ & $-0.9$ & 1.2 \\
    Phi-4 (4-bit) & Phi & 14B & 76.0 & $+2.4$ & $+1.4$ & $-0.6$ & $-1.7$ & 1.5 \\
    DeepSeek-R1-Qwen-7B & DeepSeek & 7B & 48.1 & $+1.3$ & $-0.5$ & $-0.9$ & $+0.7$ & 0.8 \\
    DeepSeek-R1-Qwen-14B (4-bit) & DeepSeek & 14B & 71.4 & $+2.6$ & $+0.8$ & $-1.6$ & $-2.1$ & 1.8 \\
    DeepSeek-R1-Llama-8B & DeepSeek & 8B & 51.6 & $+1.6$ & $+1.3$ & $-0.1$ & $+0.2$ & 0.8 \\
    \midrule
    Qwen2-VL-7B & Qwen-VL & 7B & 82.3 & $+0.5$ & $-0.2$ & $-1.0$ & $-0.9$ & 0.6 \\
    \midrule
    \multicolumn{8}{l}{\textbf{Mean $|$relative error$|$ (LLMs / VLM)}} & \textbf{1.4 / 0.6} \\
    \bottomrule
  \end{tabular}}
\end{table}

\subsection{Experimental Setup}
We test eight instruction-tuned LLMs from five families and three size classes on MMLU, three of them DeepSeek-R1 distills~\cite{deepseek2025r1}, and a vision-language model~\cite{wang2024qwen2vl} on ScienceQA. We report multiple-choice accuracy over $2{,}000$ held-out questions, $1{,}500$ for the VLM, taking the score vector to be the model's logits over the candidate answer letters.

The vision suite is our primary calibration testbed because it compares many architectures under matched conditions within each dataset. Its backbones vary in depth and design family, from ResNet~\cite{he2016deep} to VGG~\cite{simonyan2014very}, each measured by top-1 accuracy on the same 2{,}000-sample split of CIFAR-10, CIFAR-100~\cite{krizhevsky2009learning}, and ImageNet-1K~\cite{deng2009imagenet}. Accuracies under noise are averaged over 32--64 realizations. With a fixed ${\sim}2{,}000$-sample evaluation set the finite-sample error is about one point, so sub-percent table entries sit at the measurement floor rather than being exact.

\subsection{Experimental Results}
We first analyze models subjected to simulated extraction pressure. For LLMs and VLMs, \method{} perturbs the confidence scores over candidate outputs, calibrating the noise to a target accuracy. This score interface is the one the threat model assumes, not an artifact of our benchmark. As Table~\ref{tab:main-nlp-vlm} shows, \method{} fits a per-model logistic curve from one noise-free sweep and inverts it in closed form, driving every LLM to within a mean relative error of $1.4\%$ of the target, $3.4\%$ at worst, and the VLM to within $0.6\%$. Under leave-one-model-out evaluation, task-specific shared curves transfer to held-out LLM and vision models, although with higher calibration error than per-model fitting; detailed results are reported in Appendix~\ref{app:margin-ablation}.

The vision suite supplies the architectural breadth to test how well calibration holds. The closed-form inverted $\sigma$ drives measured accuracy onto the target with a mean absolute relative error of $1.1\%$ across all twenty-five vision variants, $5.0\%$ at worst. Measured accuracy decreases as the injected noise grows, tracking the monotone floor of the accuracy corridor (Corollary~\ref{cor:corridor}). Per-model results for the eleven CIFAR-100 and eight ImageNet architectures, including VGG-16 and ViT-B/16, can be found in Table~\ref{tab:vision-permodel} in the Appendix. A single shared curve is looser but still transfers leave-one-architecture-out to an unseen model, the precision-cost trade-off we quantify in the margin-statistic ablation of Section~\ref{sec:ablations}.

\begin{table}[t]
\centering
\caption{Extraction under calibrated output perturbation at 95\%, 85\%, and 75\% retention of clean teacher accuracy. \textit{Base}: student accuracy before distillation; \textit{Clean}: accuracy after distillation from the unperturbed teacher. Vision surrogates are trained from scratch (\textit{Base}: N/A). $\dagger$: surrogate accuracy $\leq$ \textit{Base}.}
\label{tab:extraction}
\footnotesize
\setlength{\tabcolsep}{2pt}%
\begin{tabular}{l ccccc @{\hspace{10pt}} l ccccc}
\toprule
\multicolumn{6}{c}{\textbf{(a) Vision (CIFAR-100)}} & \multicolumn{6}{c}{\textbf{(b) LLM (MMLU, Qwen2.5)}} \\
\cmidrule(lr){1-6}\cmidrule(lr){7-12}
 & \textbf{base} & \textbf{clean} & \textbf{95\%} & \textbf{85\%} & \textbf{75\%} & & \textbf{base} & \textbf{clean} & \textbf{95\%} & \textbf{85\%} & \textbf{75\%} \\
\midrule
\textit{Teacher} & --- & \textit{64.0} & \textit{60.7} & \textit{54.3} & \textit{47.2} & \textit{Teacher} & --- & \textit{72.0} & \textit{68.2} & \textit{61.2} & \textit{53.1} \\
\midrule
RN-56            & --- & 59.7 & 57.8 & 54.2 & \textbf{49.3} & 0.5B & 44.1 & 51.1 & 48.9 & 48.3 & \textbf{46.4} \\
RN-110$^{\star}$ & --- & 61.0 & 59.5 & 55.1 & \textbf{50.2} & 1.5B & 59.0 & 62.1 & 61.7 & 60.0 & \textbf{58.6}\rlap{$^{\dagger}$} \\
VGG-16           & --- & 60.1 & 60.2 & 56.3 & \textbf{51.5} & 3B   & 64.2 & 67.8 & 67.9 & 65.7 & \textbf{62.9}\rlap{$^{\dagger}$} \\
\bottomrule
\end{tabular}
\end{table}

\subsection{Extraction Evaluation}
\label{sec:extraction}
The calibration above controls the exposed accuracy of the defended \emph{teacher}, the target model $f$. We now test whether that degrades a \emph{surrogate} distilled from it. The attacker queries same-distribution data the teacher has never seen and receives possibly perturbed score vectors with no ground-truth labels, distills a surrogate from them, and is scored on a disjoint held-out set, so teacher and surrogate share one yardstick. The $\sigma$ is calibrated on the evaluation distribution exactly as deployed, following Section~\ref{sec:method}. \emph{Vision:} a ResNet-110 teacher trained on $25{,}000$ CIFAR-100 images, clean test accuracy $64.0\%$. The remaining $25{,}000$ training images form the attacker's query pool, with margin statistics indistinguishable from test, median $3.55$ versus $3.53$. Surrogates are ResNet-56, an \emph{exact-match} ResNet-110, and VGG-16, trained from scratch and scored on the official test set. \emph{LLM:} a Qwen2.5-7B teacher on MMLU, held-out accuracy $72.0\%$. The attacker submits $12{,}042$ unseen questions, and pretrained students Qwen2.5-\{0.5B, 1.5B, 3B\} are distilled via LoRA~\cite{hu2022lora} from the returned answer-score vectors and scored on $2{,}000$ held-out questions. Because a pretrained student's absolute accuracy is dominated by its prior, the quantity the defense controls there is the \emph{distillation gain} over the student's own pre-attack accuracy. A rational attacker can always early-stop, so driving the gain to zero is the strongest achievable outcome.

\begin{table}[t]
\centering
\small
\caption{\textbf{Repeated-query averaging attack.} Surrogate top-1 accuracy (\%) after averaging $M$ queries per input, in the ResNet-110 configuration of Table~\ref{tab:extraction}(a) at the 75\% retention target. \emph{Fresh i.i.d.} noise is redrawn per query, whereas \emph{deterministic per-input seeding} reuses one draw, which makes its $M{=}16$ entry a replicate of fresh $M{=}1$. This is a separate training run, so the single-draw values here and in Table~\ref{tab:extraction} ($50.0$, $50.3$, $50.2$) differ by run-to-run variation alone. Dashes indicate settings not evaluated.}
\label{tab:averaging}
\setlength{\tabcolsep}{8pt}
\begin{tabular}{l c c c c c}
\toprule
\textbf{Noise drawn} & $M{=}1$ & $M{=}2$ & $M{=}4$ & $M{=}8$ & $M{=}16$ \\
\midrule
Fresh i.i.d.\ (attack)          & 50.0 & 54.3 & 56.9 & 59.2 & 60.4 \\
Deterministic seeding (defense) & ---  & ---  & ---  & ---  & 50.3 \\
\bottomrule
\end{tabular}
\end{table}

The calibrated perturbation transfers from teacher to surrogate in both modalities (Table~\ref{tab:extraction}). In the vision domain, every surrogate degrades with the throttle, remaining within measurement noise at the 95\% setting, and tracks the throttled teacher. At the 85\% setting, the three surrogates score $54.2$, $55.1$, and $56.3\%$ against the teacher's $54.3\%$, and at 75\% they follow the teacher down to $49.3$, $50.2$, and $51.5\%$ as it falls to $47.2\%$. Cross-sample distillation averages away roughly 40\% of the teacher's accuracy drop at the 75\% setting, but a residual penalty survives in every surrogate, even at a single query per input. Neither the exact-match ResNet-110$^{\star}$ nor the 8 times larger VGG-16 circumvents the defense. At the 75\% setting they fall 10.8 and 8.6 points respectively. On LLMs, the defense similarly regulates the distillation gain. Undefended distillation adds $+7.0$, $+3.1$, and $+3.6$ points over the students' base accuracy, while mild throttling shrinks this improvement. At the 75\% setting the gain of the two stronger students is erased outright, yielding $-0.4$ and $-1.3$ points for the entries marked $^{\dagger}$, so $12{,}042$ queries against the throttled API buy those two students nothing. Given the evaluation sample sizes, the clean-to-75\% drop is about $3\sigma$, so the milder per-step differences should be read as a consistent trend rather than as individually significant.

An attacker who knows the defense can try to cancel the zero-mean noise by averaging $M$ repeated queries per input, paying ${\sim}M\times$ the queries. Table~\ref{tab:averaging} measures the trade-off on a ResNet-110 teacher calibrated to retain $75\%$ of its clean accuracy. Under \emph{fresh} i.i.d.\ noise the surrogate is averaged back toward this run's undefended $60.5\%$, so recovering an essentially clean surrogate costs ${\sim}16\times$ the budget and even a merely usable one ${\sim}4$--$8\times$. Drawing the noise \emph{deterministically} per input removes the free repeats, holding the surrogate at $50.3\%$ even at $M{=}16$. The attacker must then source $M$ distinct near-duplicate queries per example, which is strictly more expensive.

If the endpoint returns only the predicted label, the provider can still inject the noise internally to flip a fraction of labels. Driving a ResNet-110 teacher to ${\sim}47\%$ this way, a standard black-box attacker with one query per image distills only $44.9\%$, against a $60.5\%$ undefended hard-label ceiling. Majority-voting over $M$ repeated queries recovers the flipped labels, but again only at the ${\sim}M\times$ query cost, label accuracy climbing from $47.3\%$ at $M{=}1$ to $63.1\%$ at $M{=}64$ and surrogate accuracy from $44.9\%$ to $56.0\%$. The cost multiplier is thus not specific to the soft-label interface.

\begin{figure}[!t]
  \centering
  \includegraphics[width=\textwidth]{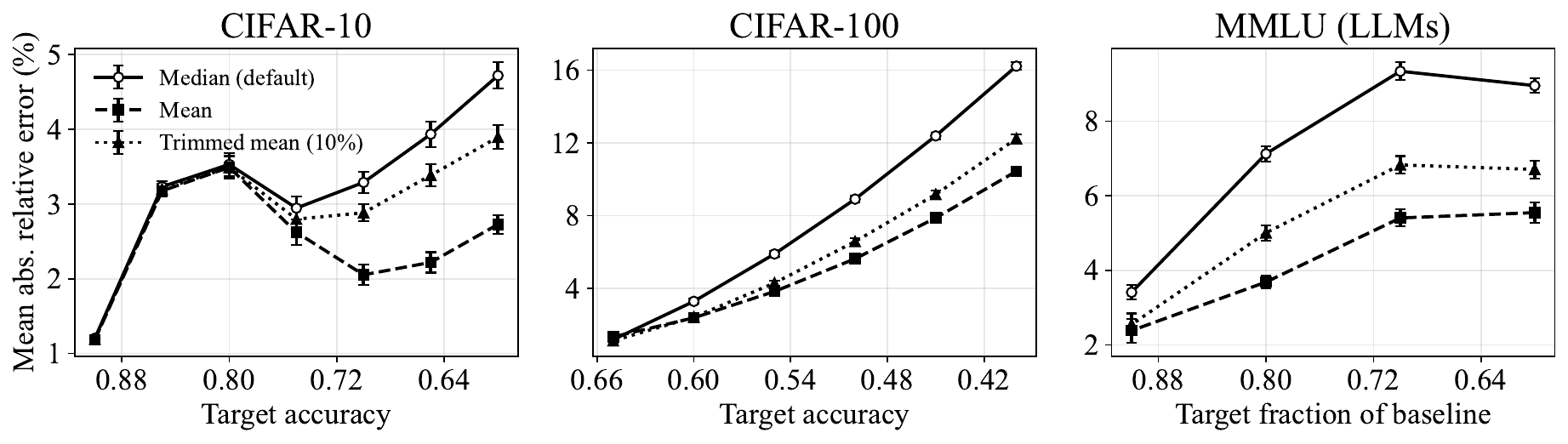}
  \caption{Margin-statistic ablation under \emph{amortized} calibration using one shared logistic curve per task. Error bars denote 95\% confidence intervals over 12 independent resampling of the calibration probe.}
  \label{fig:ablation}
\end{figure}

\subsection{Ablations and Analysis}
\label{sec:ablations}

Within the margin-normalization step, the principal design choice is the statistic used to define $S$. It matters only in the amortized regime, since a per-model fit absorbs any rescaling. On CIFAR-10, CIFAR-100, and MMLU the median, mean, and $10\%$-trimmed mean all hold the shared-curve error in single digits, with the mean tightest at $2.5$, $5.2$, and $4.3\%$ and the default median loosest at $3.3$, $8.0$, and $7.2\%$. As Figure~\ref{fig:ablation} shows, a provider sharing curves at scale recovers up to three points by switching to the mean. The per-model breakdown and held-out transfer appear in Appendix~\ref{app:margin-ablation}.

A per-model fit of Equation~\eqref{eq:logistic-decay} attains $R^2 \approx 0.998$ across the vision architectures, and its closed-form monotone inversion is what yields the $1.1\%$ error reported above. Its flat head and sharp transition match how accuracy decays on discrete decision tasks, which exponential-family alternatives over-smooth (Section~\ref{sec:method}).

\section{Discussion}

\paragraph{Design rationale.} We intervene on logits because a softmax decision depends on logit differences. Adding i.i.d.\ Gaussian noise reshapes that difference distribution without touching any learned parameter, corrupting the soft-label signal a surrogate trains on by a calibrated amount the provider sets, at a cost of $O(K)$ per query from $K$ random draws and an element-wise addition.

\paragraph{Limitations.} \method controls the utility an attacker distills from an API, not orthogonal properties such as safety, bias, or robustness. It works best when the API returns confidence scores, and reaches only the endpoints the provider still tracks~\cite{ren2026residual}. For hard-label APIs, the mechanism can still increase an attacker's query cost, but it necessarily reduces the accuracy observed by legitimate users because no non-top-1 scores can be perturbed without changing the returned label. Margin normalization weakens under highly imbalanced scores, and the operating range must be validated so extreme noise stays well-behaved. Our evaluation of more than thirty model-dataset combinations does not cover every domain, and its breadth sits where the threat is thinnest. The broadest cross-model study uses image classifiers, whereas documented campaigns increasingly target LLM APIs, so the vision results read primarily as evidence of cross-model calibration transfer.

As for the guarantees, Proposition~\ref{prop:monotone} controls agreement with the clean model, so accuracy control holds only up to the corridor of Corollary~\ref{cor:corridor} plus the empirical fit. Proposition~\ref{prop:cost} covers per-input recovery rather than extraction as a whole, for the reasons Section~\ref{sec:cost-guarantee} sets out. Input-seeded noise reaches its limit at the same place. It removes exact repeats, but a determined attacker can still manufacture fresh draws by paraphrasing or near-duplicating inputs and average the noise out.

Our study also fixes one query regime, standard distillation from in-dis\-tri\-bu\-tion inputs at a single temperature, while attacks in the wild draw from out-of-dis\-tri\-bu\-tion or synthetic data and select actively~\cite{orekondy2019knockoff,pal2020activethief,juuti2019prada}. Active, uncertainty-based selection is unlikely to help, since it targets exactly the low-margin points our noise disrupts hardest. At the evaluated 75\% retention setting the lowest-margin quintile has its predicted labels flipped at 1.5--1.7 times the overall rate, 59.8\% versus 39.1\% on MMLU and 72.8\% versus 42.6\% on CIFAR-100. Temperature rescales the perturbed logits uniformly, so the corruption of the soft targets should survive whatever temperature the attacker distills at.

\method{} is therefore a cost-amplification mechanism, not a denial one, and works best alongside query-volume monitoring or rate limits that bound the repeat factor $M$, and post hoc ownership verification.

\paragraph{Ethics.} Throttling acts on detector-flagged accounts, so a misclassified legitimate user may receive degraded output. What that user loses is set by the retention target, up to calibration error and distribution shift. At a $75\%$ target the account keeps three quarters of the clean model's accuracy, and unmodified service resumes once the flag is lifted. The burden still falls unevenly. A client that reads only top-1 labels sees a change only when the prediction flips, whereas one that consumes full score vectors sees perturbed values at much milder settings and has the greater need for an appeal mechanism. Graduated degradation is gentler than a ban, but the provider still has to audit the detector for disparate impact and keep throttling out of safety-critical deployments.

\paragraph{Future work.} The most direct extension is from scored decisions to free-form generation, where perturbing every token may compound the error but sequence-level quality needs its own evaluation. \emph{Adaptive throttling} could raise the noise as a suspected account's cumulative query count grows, with the ramp driven by drift in a tracked signal~\cite{ma2026cliend}. On the theory side, a bound on the surrogate quality reachable at a given degradation level would complement Proposition~\ref{prop:cost}, though a worst-case version depends on convex relaxations of the surrogate's activations~\cite{ma2025convexhull}. Empirically, extraction on further datasets and variance estimates across training seeds remain open.
\section{Related Work}

Model extraction builds on knowledge distillation~\cite{hinton2015distilling}. Tram\`er et al.~\cite{tramer2016stealing} carried this into the black-box setting, fitting a surrogate on the scores a prediction API returns. Later attacks cloned from out-of-distribution images~\cite{orekondy2019knockoff} or synthesized queries~\cite{truong2021datafree}, and Jagielski et al.~\cite{jagielski2020high} showed high-accuracy extraction is cheap even when exact recovery is hard. The threat now reaches language models, from BERT APIs~\cite{krishna2020thieves} to recovering a production LLM's output projection~\cite{carlini2024stealing}, where querying stays far cheaper than training.

Defenses split into reactive and proactive. Detection flags query streams that deviate from benign use~\cite{juuti2019prada,zhang2021seat}, while watermarking and fingerprinting pursue after-the-fact attribution, from trigger sets~\cite{adi2018turning} to dataset inference~\cite{maini2021dataset} and decoding-time LLM watermarks~\cite{kirchenbauer2023watermark} that survive distillation yet are erasable by paraphrasing~\cite{pan2025watermark}. Detection can only decide whether to cut an account off, and attribution arrives once the surrogate has already shipped. Proactively, AlgoSpec obscures released data so only an authorized algorithm can use it~\cite{liu2024transparent}, non-transferable examples recode inputs into a model-specific subspace~\cite{wang2025catchonlyone}, and key-based usage control locks utility in the weights behind an owner's key~\cite{wang2026adaloc}. All three assume the defender owns the inputs or ships the weights, neither of which holds against an attacker who only queries an API.

Output perturbation instead controls what the surrogate learns, serving a suspect plausible-but-corrupted scores. Reverse-sigmoid perturbation~\cite{lee2019defending} introduced the idea. Prediction poisoning~\cite{orekondy2020prediction} and gradient redirection~\cite{mazeika2022steer} steer the surrogate's gradient under a utility budget. Adaptive misinformation~\cite{kariyappa2020adaptive} and diverse ensembles~\cite{kariyappa2021protecting} mislead out-of-distribution queries, while undistillable training reshapes the output so any student degrades~\cite{ma2021undistillable}. Recent LLM variants poison reasoning traces~\cite{savani2025antidistillation}, fine-tune the output layer~\cite{li2025doge}, or bound distillable information~\cite{fang2026distillation}. A parallel line raises the query cost of extraction~\cite{dziedzic2022increasing,dubinski2023bucks}, a goal \method{} shares on the output side.

Yet in all of these output-perturbation defenses the map from perturbation strength to surrogate accuracy is architecture-dependent, so strength must be set per model by hand and does not scale to hundreds of variants under latency constraints. \method{} absorbs that architecture-specific scale into the squared median logit margin, normalizing the noise variance by it and fitting a monotone logistic curve that inverts in closed form. We build on the i.i.d.\ Gaussian logit-noise primitive and order-preservation analysis of Wang et al.~\cite{wang2025aim}, introduced there for training-free utility tiering with per-model tuning, a non-security use. We repurpose it against distillation-based extraction and add the cross-architecture calibration that removes the per-model cost.

\section{Conclusion}
We presented \method, which lets a score-returning API control the utility it exposes to a suspected extraction client. The control rests on one regularity, the monotone logistic law that every architecture we evaluate follows once the noise is normalized by the logit margin. \method inverts it in closed form to hit any target within ${\sim}1\%$ per model, or at single-digit held-out error across architectures. Undoing the zero-mean perturbation on a single input costs several times the query budget, a multiplier we lower-bound in closed form. What a provider previously hand-tuned per model it can now set once, from a target, on any architecture it serves.

\appendix
% Keep hyperref anchors distinct from the body sections (llncs + \appendix
% otherwise reuses section.1, section.2, ... and links land in the wrong place).
\renewcommand{\theHsection}{app.\Alph{section}}

\section{Ablations}
\label{app:margin-ablation}
This section expands the margin-statistic ablation of Figure~\ref{fig:ablation}. Each seed re-estimates $S$ from a fresh one-shot $2{,}000$-sample probe, a bootstrap resample of margins for the probe-sized LLM evaluation set. The task-level ordering reported in Section~\ref{sec:experiments} holds at the per-model level, the mean beating the median on five of six CIFAR-10 models, seven of eleven CIFAR-100 models, and seven of eight LLMs. Seed-to-seed variability is negligible for every statistic (95\% CIs of ${\pm}0.1$--$0.3$ points), so the median's robustness-to-outliers motivation carries no measurable cost at $N{=}2{,}000$. These are in-sample shared fits that isolate the statistic.

\paragraph{Held-out transfer (leave-one-architecture-out).} For each architecture we fit the shared logistic curve on the other architectures of the task and calibrate the held-out one from its margin $S$ alone. Held-out error stays within a point of the in-sample fit on every task, $3.7\%$ versus $3.3\%$ on CIFAR-10 (6 architectures), $8.8\%$ versus $8.0\%$ on CIFAR-100 (11), and $8.2\%$ versus $7.2\%$ on MMLU (8), so the curve generalizes to unseen architectures. Transfer is uneven per architecture. Holding out a ResNet on CIFAR-100 costs ${\sim}12\%$, the reference set then being dominated by other families, while VGG-13 and RepVGG-A0 cost ${\sim}1$--$2\%$. On MMLU the Qwen models are hardest to predict held-out (${\sim}11$--$13\%$) and the Mistral and Phi models the easiest (${\sim}3\%$).

\FloatBarrier

\section{Omitted Proofs}
\label{app:proofs}

\begin{proof}[of Proposition~\ref{prop:monotone}]
Write \(\epsilon_k=\sigma\xi_k\) with \(\xi_k\sim\mathcal{N}(0,1)\) i.i.d. Up to the probability-zero event of ties, \(\{\operatorname*{arg\,max}_k(z_k+\epsilon_k)=a\}=\{\xi_j-\xi_a<g_j/\sigma\ \forall j\neq a\}\). Conditioning on \(\xi_a=u\) and using independence of \(\{\xi_j\}_{j\neq a}\) gives \(\Pr[\,\cdot\mid \xi_a=u]=\prod_{j\neq a}\Phi(g_j/\sigma+u)\); integrating against \(\phi\) yields \eqref{eq:top1-exact}. For monotonicity, the event is \(\{\boldsymbol{\xi}\in R(\sigma)\}\) with \(R(\sigma)=\{\boldsymbol{\xi}:\xi_j-\xi_a<g_j/\sigma\ \forall j\neq a\}\). Since every \(g_j>0\), each threshold \(g_j/\sigma\) is strictly decreasing in \(\sigma\), so \(R(\sigma')\subsetneq R(\sigma)\) for \(\sigma'>\sigma\); the Gaussian law assigns the nonempty difference positive measure, hence \(\pi(\sigma')<\pi(\sigma)\). The integrand of \eqref{eq:top1-exact} is continuous in \(\sigma\) and dominated by \(\phi\), so dominated convergence gives both continuity of \(\pi\) on \((0,\infty)\) and its limits: as \(\sigma\to0^+\) all thresholds diverge and \(\pi\to1\); as \(\sigma\to\infty\) all thresholds vanish and \(\pi\to\Pr[\xi_j<\xi_a\ \forall j\neq a]=1/K\) by exchangeability of i.i.d.\ continuous variables. \qed
\end{proof}

\begin{proof}[of Corollary~\ref{cor:corridor}]
Split the evaluation set into the correctly classified inputs (fraction \(A_0\)) and the misclassified ones (fraction \(1-A_0\)). On a correct input the event \(\{\operatorname*{arg\,max}_k(z_k+\epsilon_k)=y\}\) is exactly the label-preservation event of Proposition~\ref{prop:monotone}, with probability \(\pi(\sigma)\); averaging gives the contribution \(A_0\bar\pi_c(\sigma)\), which is continuous and strictly decreasing from \(A_0\) to \(A_0/K\) termwise by Proposition~\ref{prop:monotone}. On a misclassified input the probability of predicting \(y\) is nonnegative, and \(\{\operatorname*{arg\,max}_k(z_k+\epsilon_k)=y\}\subseteq\{z_y+\epsilon_y>z_a+\epsilon_a\}\); since \(\epsilon_y-\epsilon_a\sim\mathcal{N}(0,2\sigma^2)\) and \(h=z_a-z_y>0\), this event has probability \(\Phi(-h/(\sqrt2\sigma))<1/2\). Averaging over the misclassified inputs yields the upper envelope. Finally, \(A(\sigma)\to1/K\) as \(\sigma\to\infty\) by the exchangeability argument of Proposition~\ref{prop:monotone}. \qed
\end{proof}

\begin{proof}[of Proposition~\ref{prop:cost}]
(i) Write the responses as \(\boldsymbol z^{(m)}=\boldsymbol z+\boldsymbol\epsilon^{(m)}\), \(m=1,\dots,M\). Their sample mean \(\bar{\boldsymbol z}\sim\mathcal{N}(\boldsymbol z,(\sigma^2/M)I_K)\) is a sufficient statistic for \(\boldsymbol z\), so it suffices to bound estimation of a Gaussian mean with per-coordinate variance \(\sigma^2/M\). For any estimator \(\hat{\boldsymbol z}\), the worst-case risk dominates the Bayes risk under the prior \(\boldsymbol z\sim\mathcal{N}(0,\rho^2 I_K)\), which is minimized by the posterior mean at per-coordinate risk \(\rho^2\sigma^2/(M\rho^2+\sigma^2)\). Letting \(\rho\to\infty\) gives \(\sup_{\boldsymbol z}\frac1K\mathbb{E}\|\hat{\boldsymbol z}-\boldsymbol z\|^2\ge\sigma^2/M\), which \(\bar{\boldsymbol z}\) attains, so \(\sigma^2/M\) is the minimax risk. Dividing by \(S^2\), the residual perturbation of any recovered logit vector carries worst-case effective intensity at least \(v/M\); since \(v/M\le v_{\mathrm{use}}\) iff \(M\ge v/v_{\mathrm{use}}\), the bound follows. (ii) With a per-input deterministic seed the \(M\) draws coincide, so their average equals a single draw of intensity \(v\). \qed
\end{proof}

\section{Per-Model Calibration Results}
Table~\ref{tab:vision-permodel} reports per-model calibration for the nineteen CIFAR-100 and ImageNet architectures, listing for each the closed-form inverted noise level $\sigma$ and the accuracy it produces at every utility target. Within each row $\sigma$ increases as the target decreases, matching the strict monotonicity of Proposition~\ref{prop:monotone}. Over these nineteen the measured accuracy lands on the target with a mean relative error of $0.8\%$; the $1.1\%$ quoted in Section~\ref{sec:experiments} pools them with the six CIFAR-10 architectures, whose per-model figures we omit for space.
\begin{table}[!t]
  \centering
  \footnotesize
  \caption{\textbf{Vision per-model calibration.} Each architecture is fit its own logistic curve. Cells report the closed-form inverted $\sigma$ and the measured top-1 accuracy, averaged over 64 noise realizations, per-entry std ${\le}0.5\%$.}
  \label{tab:vision-permodel}
  \setlength{\tabcolsep}{2pt}
  \renewcommand{\arraystretch}{1}
  \begin{tabular}{l c cccccc}
    \toprule
    \multicolumn{2}{c}{\textit{(a) CIFAR-100}} & \multicolumn{6}{c}{$\sigma$\,/\,measured top-1 (\%) at target accuracy (\%)} \\
    \cmidrule(lr){3-8}
    \textbf{Model} & \textbf{Base} & \textbf{65} & \textbf{60} & \textbf{55} & \textbf{50} & \textbf{45} & \textbf{40} \\
    \midrule
    ResNet-20   & 66.7 & 0.70/64.8 & 1.44/59.6 & 1.99/54.7 & 2.50/49.8 & 3.01/45.1 & 3.55/40.2 \\
    ResNet-32   & 68.4 & 1.23/64.7 & 1.95/59.5 & 2.52/54.8 & 3.07/50.0 & 3.61/45.1 & 4.19/40.2 \\
    ResNet-44   & 68.7 & 1.38/64.6 & 2.13/59.7 & 2.74/54.9 & 3.30/50.0 & 3.88/45.3 & 4.48/40.4 \\
    ResNet-56   & 69.3 & 1.59/64.5 & 2.33/59.5 & 2.95/54.8 & 3.54/50.0 & 4.12/45.2 & 4.75/40.3 \\
    VGG-13      & 70.9 & 2.11/64.6 & 2.76/59.9 & 3.31/55.1 & 3.82/50.3 & 4.34/45.3 & 4.87/40.3 \\
    VGG-16      & 70.3 & 1.99/65.8 & 2.59/61.0 & 3.08/55.7 & 3.52/50.2 & 3.96/44.7 & 4.41/39.4 \\
    VGG-19      & 67.0 & 1.33/65.1 & 2.16/60.7 & 2.72/55.6 & 3.20/50.2 & 3.65/44.8 & 4.12/39.7 \\
    MobileNetV2 & 72.9 & 2.14/64.7 & 2.70/59.9 & 3.20/55.1 & 3.68/50.2 & 4.17/45.2 & 4.69/40.2 \\
    RepVGG-A0   & 73.0 & 2.37/64.6 & 2.99/59.9 & 3.55/55.2 & 4.08/50.3 & 4.63/45.3 & 5.21/40.2 \\
    RepVGG-A1   & 74.2 & 2.59/64.7 & 3.17/60.0 & 3.70/55.2 & 4.21/50.2 & 4.73/45.2 & 5.29/40.1 \\
    RepVGG-A2   & 75.2 & 2.61/64.8 & 3.14/60.1 & 3.63/55.2 & 4.11/50.3 & 4.59/45.3 & 5.11/40.1 \\
    \midrule
    \multicolumn{2}{c}{\textit{(b) ImageNet-1K}} & \multicolumn{6}{c}{} \\
    \textbf{Model} & \textbf{Base} & \textbf{70} & \textbf{60} & \textbf{50} & \textbf{40} & \textbf{30} & \textbf{20} \\
    \midrule
    ResNet-18       & 71.8 & 1.11/68.8 & 2.34/59.3 & 3.15/50.5 & 3.92/40.8 & 4.81/30.5 & 6.00/19.8 \\
    ResNet-50       & 80.9 & 2.89/70.0 & 3.66/60.2 & 4.32/50.1 & 5.00/40.1 & 5.80/29.9 & 6.89/19.9 \\
    VGG-16          & 72.2 & 0.80/69.7 & 2.03/59.4 & 2.98/49.9 & 3.98/40.4 & 5.21/30.5 & 6.94/20.1 \\
    MobileNetV3-L   & 75.8 & 1.31/68.5 & 1.85/60.4 & 2.27/51.0 & 2.68/40.7 & 3.15/29.9 & 3.77/19.5 \\
    EfficientNet-B0 & 78.3 & 1.49/69.0 & 1.96/60.7 & 2.35/51.0 & 2.73/40.5 & 3.18/29.9 & 3.78/19.6 \\
    DenseNet-121    & 76.2 & 1.22/68.1 & 1.76/59.9 & 2.18/51.0 & 2.60/40.9 & 3.09/30.3 & 3.75/19.6 \\
    ConvNeXt-T      & 84.5 & 1.90/71.0 & 2.23/60.5 & 2.53/49.9 & 2.83/39.6 & 3.18/29.6 & 3.67/20.1 \\
    ViT-B/16        & 85.7 & 1.88/71.2 & 2.19/60.7 & 2.47/49.9 & 2.76/39.4 & 3.11/29.6 & 3.58/20.0 \\
    \bottomrule
  \end{tabular}
\end{table}

\bibliographystyle{splncs04}
\bibliography{reference}

\end{document}